\documentclass{article}

\PassOptionsToPackage{numbers, compress}{natbib}
\usepackage[preprint]{neurips_2026}

\usepackage[utf8]{inputenc} 
\usepackage[T1]{fontenc}    
\usepackage{hyperref}       
\usepackage{url}            
\usepackage{booktabs}       
\usepackage{amsfonts}       
\usepackage{nicefrac}       
\usepackage{xcolor}         
\RequirePackage{graphicx}
\usepackage{xspace}
\usepackage{graphicx}
\usepackage{subfigure}
\usepackage{booktabs}
\usepackage[shortlabels]{enumitem}
\usepackage{pifont}
\usepackage[dvipsnames]{xcolor}

\RequirePackage{fancyhdr}
\RequirePackage{xcolor}
\RequirePackage{algorithm}
\RequirePackage{algorithmic}
\RequirePackage{natbib}
\RequirePackage{eso-pic}
\RequirePackage{forloop}
\RequirePackage{url}
\usepackage{hyperref}
\usepackage{booktabs}
\usepackage{amsmath,amssymb,amsthm}
\usepackage{soul}
\usepackage{subcaption}
\usepackage{multirow}
\usepackage{mathtools}
\usepackage{amsthm}
\usepackage[thinc]{esdiff}
\usepackage[table]{xcolor}
\definecolor{red}{rgb}{0.921569, 0.282353, 0.247059}
\definecolor{green}{rgb}{0.501961, 0.792157, 0.239216}
\definecolor{orange}{rgb}{0.937255, 0.529412, 0.2}
\usepackage{longtable}
\usepackage[capitalize,noabbrev]{cleveref}
\usepackage[textsize=tiny]{todonotes}
\setcitestyle{square}

\theoremstyle{plain}
\newtheorem{theorem}{Theorem}[section]

\theoremstyle{definition}

\theoremstyle{remark}

\title{LiNC: Lightweight Noise Correction \\ via Per-Sample Trust and Gaussian Mixture Modeling}

\author{%
  Abhishek Moturu \\ 
  Department of Computer Science\\
  University of Toronto\\
  The Hospital for Sick Children \\
  UHN KITE Research Institute \\
  T-CAIREM \\
  Vector Institute \\
  \texttt{moturuab@cs.toronto.edu} \\
  \And
    Babak Taati \thanks{equal contribution} \\ 
  Department of Computer Science\\
  Institute of Biomedical Engineering \\
  University of Toronto\\
  Rehabilitation Sciences Institute \\
  UHN KITE Research Institute \\
  Vector Institute \\
  \texttt{taati@cs.toronto.edu} \\
  \AND
  Anna Goldenberg $^*$  \\
  Department of Computer Science\\
  Department of Laboratory Medicine and Pathobiology \\
  University of Toronto\\
  The Hospital for Sick Children \\
  T-CAIREM \\
  Vector Institute \\
  \texttt{anna.goldenberg@utoronto.ca} \\
}

\begin{document}

\newcommand{\synpain}{Syn\textsc{Pain}\xspace}
\newcommand{\gaitgen}{\textsc{Gait}Gen\xspace}
\maketitle
\setcounter{footnote}{0}


\begin{abstract}
Label noise is common in medical imaging datasets due to factors such as inter-rater variability, annotation errors, and ambiguous cases. This can severely undermine the reliability and clinical effectiveness of machine learning models trained using those datasets. To address this challenge, we introduce Lightweight Noise Correction (LiNC), which adds a single trainable trust parameter per training sample and learns when to use the observed label and when to defer to the model during a standard training loop. The key idea is to train using a convex combination of the observed label and the model's own predictive distribution, controlled by a per-sample trust parameter. We show that the gradient of this objective drives trust values in opposite directions for clean versus noisy samples in the early training phase, yielding separable trust distributions. We use a 3-component Gaussian Mixture Model over the trust values to separate them into clean, ambiguous, and noisy cases and then execute a short soft-correction phase on the noisy cases and a final hard correction phase. Experiments on ten 2D datasets from MedMNISTv2 under label noise of up to 50\% show consistent gains in accuracy and strong mislabel detection. LiNC adds negligible asymptotic overhead: the training-time complexity remains dominated by the base network, with additional memory growing linearly with the size of the training set.
\end{abstract}

\section{Introduction}
\label{sec:intro}
Within healthcare, medical imaging is essential in supporting clinical tasks such as diagnosis, treatment planning, and disease monitoring. Recent advancements in deep learning have significantly improved medical image analysis by automating the detection and classification of various medical conditions. These advancements heavily depend on the availability of accurately labeled datasets. Machine learning models tend to severely degrade in performance when trained on noisy data. On the other hand, label noise is prevalent in healthcare datasets due to inconsistent annotations, human errors, and ambiguous findings during the annotation process. Label noise can change the optimization landscape, hurt calibration, and amplify spurious correlations, if not properly addressed \citep{karimi2020deep,zhang2016understanding,zhang2021understanding,arpit2017closer,guo_calibration_2017}.

Many studies focus on learning with noisy labels in natural images \citep{patrini2017making,han2018co,reed2014training,li2020dividemix,zhang2018generalized,wang2019symmetric}. However, medical imaging has additional constraints: clean validation sets are expensive, data distributions can shift, and practitioners may need interpretability into which training labels are unreliable. Methods that require additional models, extensive hyperparameter tuning, or dataset-specific or task-specific thresholds are hard to justify in clinical workflows, which are often time-constrained and resource-constrained.

We ask: \emph{Can we get label correction and label noise signals essentially ``for free'' while keeping training close to standard fine-tuning?} For this, we introduce LiNC, which uses a trainable trust parameter $\alpha_i$ per sample that decides how much the model should trust the observed label versus its own predictive distribution. Intuitively, if the model consistently assigns low probability to the observed label for a sample, the gradient pushes $\alpha_i$ down, shifting supervision toward the model prediction and if the model assigns high probability to the observed label, $\alpha_i$ increases, preserving supervision. This produces a natural separation between clean and noisy samples without any access to true ground-truth labels.

\paragraph{Contributions.} LiNC is a lightweight noise correction method that does not need more models, clean validation sets, or pruning, yet yields three concrete benefits:
\begin{itemize}
    \item We derive a simple expression for the per-sample trust gradient, $\partial\mathcal{L}/\partial \alpha_i$, and show how it induces separability between clean and noisy labels.
    \item We use an unsupervised 3-component GMM over trust values (to separate noisy / ambiguous / clean samples) to get threshold-free separation and correction, inspired by GMM-based separation in noisy-label learning \citep{li2020dividemix}.
    \item We add $\mathcal{O}(N)$ memory and a negligible constant-factor compute cost, while producing per-sample trust scores that can be used for dataset audits.
\end{itemize}

\begin{figure*}
    \centering
    \includegraphics[scale=0.34]{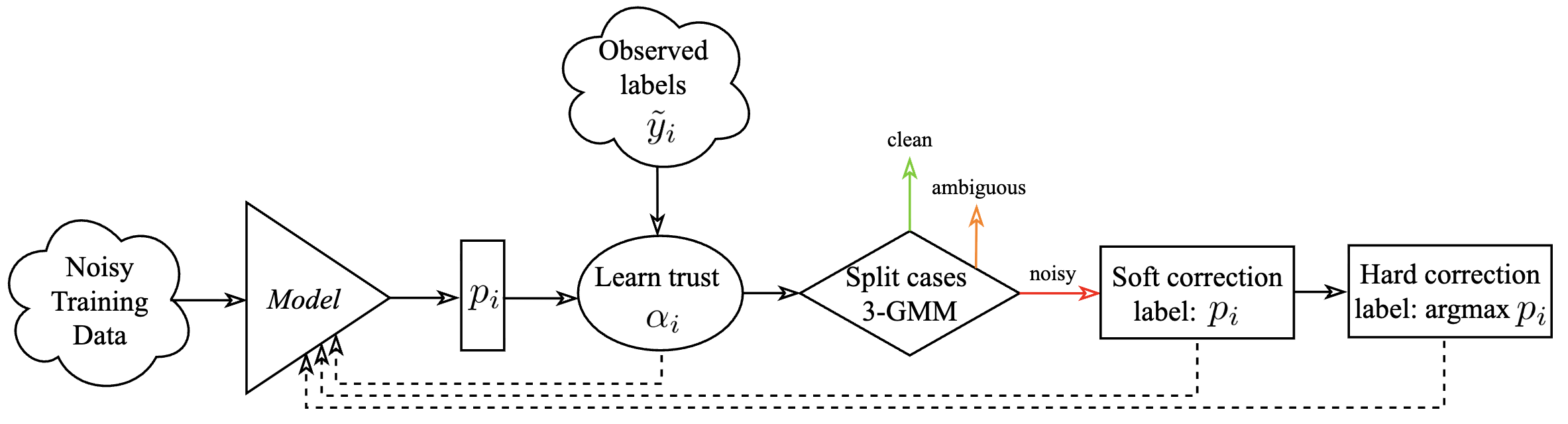}
    \caption{\textbf{Overview of LiNC.} The model produces a class-probability vector $p_i$ for each training sample. During soft warm-up, LiNC jointly learns a per-sample trust parameter $\alpha_i$ and trains using the target
$q_i = (1-\alpha_i)\operatorname{stopgrad}(p_i) + \alpha_i e_{\tilde{y}_i}$,
which interpolates between the model prediction and the observed label $\tilde{y}_i$. A three-component Gaussian mixture model (3-GMM) is then fitted to the learned trust values, with components ordered by their means and interpreted as noisy, ambiguous, and clean. Only samples with the lowest trust (i.e. the noisy component) are corrected: they first use $p_i$ as a soft target and are subsequently assigned the hard label $\arg\max p_{i}$ for final training. Clean and ambiguous samples are not relabeled, reducing the risk of over-correcting borderline cases. Dashed arrows indicate that the phase-specific supervision is fed back into successive updates of the same model.}
    \label{linc}
\end{figure*}

\section{Related Work}
\label{sec:relatedwork}

Classical approaches assume the existence of a noise transition matrix, or an approximate one, and perform forward/backward loss correction \citep{patrini2017making} or explicitly learn a noise model \citep{goldberger2017training}. Robust loss functions aim to reduce sensitivity to incorrect labels without explicitly correcting them, e.g., generalized cross entropy \citep{zhang2018generalized}, symmetric cross entropy \citep{wang2019symmetric}, and robust losses \citep{ghosh2017robust}. These methods are simple but do not directly provide an interpretable noise signal.

Given the empirical observation that deep networks fit clean data before memorizing noise \citep{arpit2017closer}, Co-teaching selects small-loss samples using two networks \citep{han2018co} and DivideMix models losses with a GMM and treats training as semi-supervised learning via MixMatch \citep{li2020dividemix,berthelot2019mixmatch}. These methods can be effective but typically require multiple networks, excessive tuning, or additional losses.

Pseudo-labeling \citep{lee2013pseudo} and bootstrapping \citep{reed2014training} blend observed labels with model predictions and progressive self label correction reduces confirmation bias \citep{wang2021proselflc,wei2020combating,song2019selfie}. LiNC has similarities to bootstrapping, but is different in a very important way: trust is learned per sample rather than fixed globally, and we provide an explicit noise separation procedure that avoids hard-coded thresholds.

Several techniques estimate label quality post-hoc using training dynamics (AUM \citep{pleiss_identifying_2020}, forgetting events \citep{toneva_empirical_2019}, DataMaps \citep{swayamdipta_dataset_2020}, EL2N/GraND \citep{paul_deep_2023}, Data-IQ \citep{seedat_dissecting_2024}, VoG \citep{agarwal_estimating_2022}, CNLCU-S \citep{xia_sample_2021}). LiNC produces an interpretable noise signal as a byproduct of training and directly uses it for correction.

\section{Method}
\label{sec:method}
\subsection{Problem Setup}
Let $\mathcal{D}=\{(x_i,\tilde{y}_i)\}_{i=1}^N$ denote the training set, where $x_i\in\mathcal{X}$ has potentially noisy observed labels $\tilde{y}_i\in\{1,\dots,C\}$, where $C$ is the number of classes, $N$ is the number of training samples.

We train classifier $f_\theta$ to output logits $z_i=f_\theta(x_i)$. Let $p_i=\mathrm{softmax}(z_i)$. We introduce per-sample trust parameters $\alpha_i\in(0,1)$.

\subsection{Trust-based Supervision}
For each sample, we form a soft target distribution as follows:
\begin{equation}
q_i(\alpha_i) = (1-\alpha_i)\mathrm{stopgrad}(p_i) + \alpha_i e_{\tilde{y}_i},
\label{eq:soft_target}
\end{equation}
where $e_{\tilde{y}_i}$ is the one-hot vector for the observed label and $\mathrm{stopgrad}(\cdot)$ blocks gradients through $p_i$ when we are updating $\alpha_i$. Cross-entropy loss with a soft target looks as follows:
\begin{equation}
\mathcal{L}_i(\theta,\alpha_i)= -\sum_{c=1}^{C} q_{i,c}(\alpha_i)\log p_{i,c}.
\label{eq:loss}
\end{equation}
This reduces to standard cross-entropy, when $\alpha_i=1$, and becomes self-training, with $p_i$ as the target, when $\alpha_i=0$.

\subsection{Trust Gradient for Clean vs. Noisy Labels}
We now show that during the warmup, gradient descent pushes $\alpha_i$ in opposite directions for clean vs. mislabeled samples.

\begin{theorem}[Sign of the trust gradient]
\label{thm:grad_sign}
For a fixed model output distribution $p_i$, the derivative of \ref{eq:loss} w.r.t. $\alpha_i$ is as follows:
\begin{equation}
\frac{\partial \mathcal{L}_i}{\partial \alpha_i} = \sum_{c=1}^C p_{i,c}\,\log p_{i,c} - \log p_{i,\tilde{y}_i}.
\label{eq:dldalpha}
\end{equation}
Hence, if the model agrees with the observed label, then $\partial\mathcal{L}_i/\partial\alpha_i<0$ and if the model disagrees with the observed label, then $\partial\mathcal{L}_i/\partial\alpha_i>0$.
\end{theorem}
\begin{proof}
Using \ref{eq:soft_target}, expand \ref{eq:loss}:
\(\mathcal{L}_i = -(1-\alpha_i)\sum_c p_{i,c}\log p_{i,c} - \alpha_i\log p_{i,\tilde{y}_i}.\) \\ Differentiating w.r.t. $\alpha_i$ yields: \(\frac{\partial \mathcal{L}_i}{\partial \alpha_i} = \sum_{c} p_{i,c}\,\log p_{i,c} - \log p_{i,\tilde{y}_i}\).
\begin{equation}
        \frac{\partial \mathcal{L}_i}{\partial \alpha}
        = \underbrace{-\log p_{i,\tilde{y}_i}}_{\text{NLL of observed label}} - \left(\underbrace{- \sum_{c} p_{i,c} \log p_{i,c}}_{\text{entropy } H(p_i)} \right)
\end{equation}
\begin{equation}
    = -\log p_{i,\tilde{y}_i} - H(p_i).
\end{equation}
In early training, deep neural networks tend to learn the easier and cleaner patterns first before memorizing noise \citep{arpit2017closer}. As a result, for most clean samples the model assigns relatively high probability to $\tilde{y}_i$, making $-\log p_{i,\tilde{y}_i}$ small, and the inequality $-\log p_{i,\tilde{y}_i} < H(p_i)$ holds, and for mislabeled samples, $\tilde{y}_i$ has a relatively low probability and $-\log p_{i,\tilde{y}_i} > H(p_i)$.

In other words, if the model agrees with the observed label, then we have $-\log p_{i,\tilde{y}_i} < H(p_i) \implies \frac{\partial \mathcal{L}_i}{\partial \alpha_i}<0$ and $\alpha$ increases toward $1$. And if the model disagrees with the observed label, then we have $-\log p_{i,\tilde{y}_i} > H(p_i) \implies \frac{\partial \mathcal{L}_i}{\partial \alpha_i}>0$ and $\alpha$ decreases toward $0$. Therefore, $\alpha_i$ becomes separable.
\end{proof}

This is closely related to the small-loss principle discussed by co-teaching and mixture-based methods \citep{han2018co,li2020dividemix}. 

\subsection{Separation of Clean vs. Noisy Labels}
Rather than picking a hard-coded threshold for $\alpha_i$, we fit a 3-component GMM to \(\{\alpha_i\}_{i=1}^N\) after warmup. The components are interpreted as \emph{noisy} (lowest mean), \emph{ambiguous} (middle mean), and \emph{clean} (highest mean). We use the standard expectation–maximization algorithm to fit the GMM \citep{dempster1977maximum} and obtain samples belonging to the three clusters. 

We use $K=3$ to be the number of GMM components for the separation of the trust parameters to explicitly model ambiguous samples. In preliminary experiments, $K=2$ tends to over-correct borderline cases, i.e. corrupt many correct labels, while $K=3$ yields a stable middle component to deal with ambiguous cases, similar to other mixture-based methods \citep{li2020dividemix}.

\subsection{Training Schedule}
LiNC runs in three phases, as shown in Figure~\ref{linc}:
\begin{itemize}
    \item \textbf{Soft warmup:} train using the soft target distribution (Equation \ref{eq:soft_target}) while performing manual gradient descent on $\alpha_i$.
    \item \textbf{Soft correction:} fit the GMM to find the ``noisy''  $\alpha_i$ cluster (lowest mean) to train those corresponding samples using the model predictions, while the remaining samples continue to train using the soft target distribution (Equation \ref{eq:soft_target}).
    \item \textbf{Hard correction:} correct the ``noisy'' labels by assigning them to be the $\arg\max$ of the model predictions and train using standard cross-entropy.
\end{itemize}

This makes LiNC cautious with borderline samples, whose labels may still be correct, while allowing it to confidently correct samples that are much more likely to be mislabeled. This reduces unnecessary label changes without leaving clear label errors uncorrected.

\begin{algorithm}[t]
\caption{Training with LiNC}
\label{alg:linc}
\begin{algorithmic}[1]
\STATE \textbf{Input:} training data $\mathcal{D}=\{(x_i,\tilde{y}_i)\}_{i=1}^N$, model $f_\theta$, trust parameters $\{\alpha_i\}$,\\
lr ($\eta_\alpha$) and wd ($\lambda_\alpha$) for $\{\alpha_i\}$, soft warmup epochs $w$, soft correction epochs $s$.
\STATE \textbf{Initialize:} $\alpha_i=1$
\FOR{epoch $=1$ to $E$}
    \FOR{minibatch $(x,\tilde{y},\mathrm{idx})$}
        \STATE $p \leftarrow \mathrm{softmax}(f_\theta(x))$
        \STATE $\alpha \leftarrow \alpha_{\mathrm{idx}}$
        \IF{epoch $\le w$} 
            \STATE $q \leftarrow (1-\alpha)\,\mathrm{stopgrad}(p) + \alpha\,e_{\tilde{y}}$
        \ENDIF
        \IF{$w < \text{epoch} \le w+s$}
            \STATE $q \leftarrow p$ for samples assigned to noisy GMM cluster
        \ENDIF
        \IF{$\text{epoch} > w+s$}
            \STATE $q \leftarrow e_{\hat{y}}$ where $\hat{y}=\arg\max p$
        \ENDIF
        \STATE calculate loss $\mathcal{L}(\theta,\alpha)= -\sum q(\alpha)\log p$
        \STATE update $\theta$ using optimizer step on $\nabla_\theta \mathcal{L}$
        \IF{epoch $\le w+s$}
            \STATE \hspace{1em} $\alpha_{\mathrm{idx}} \leftarrow \alpha_{\mathrm{idx}} - \eta_\alpha(\nabla_{\alpha_{\mathrm{idx}}}\mathcal{L} + \lambda_\alpha \alpha_{\mathrm{idx}})$
        \ENDIF
    \ENDFOR
    \IF{epoch $= w$}
        \STATE fit 3-component GMM to $\{\alpha_i\}$ \& identify noisy cluster (lowest mean)
    \ENDIF
    \IF{epoch $= w+s$}
        \STATE for $x_i \in \mathcal{D}$, correct labels $\hat{y}_i = \arg\max f_\theta(x_i)$ for noisy cluster
        \STATE freeze $\alpha_i$
    \ENDIF
\ENDFOR

\STATE \textbf{Return:} trained model $f_\theta$ and trust parameters $\{\alpha_i\}$
\end{algorithmic}
\end{algorithm}

\subsection{Complexity: With vs. Without LiNC}
\label{sec:complexity}
\paragraph{Runtime.} Let $E$ be the number of training epochs, $B$ be batch size, and $T_{\text{fwd}}(B)$ and $T_{\text{bwd}}(B)$ denote the time required for one forward and backward pass on batch size $B$, respectively. 

Standard training costs $T_{\text{base}} = E\cdot\frac{N}{B}\cdot (T_{\text{fwd}}(B) + T_{\text{bwd}}(B))$. 
This gives us $\mathcal{O}(EN)$ runtime without LiNC.

\begin{figure*}
\centering
\includegraphics[scale=0.3]{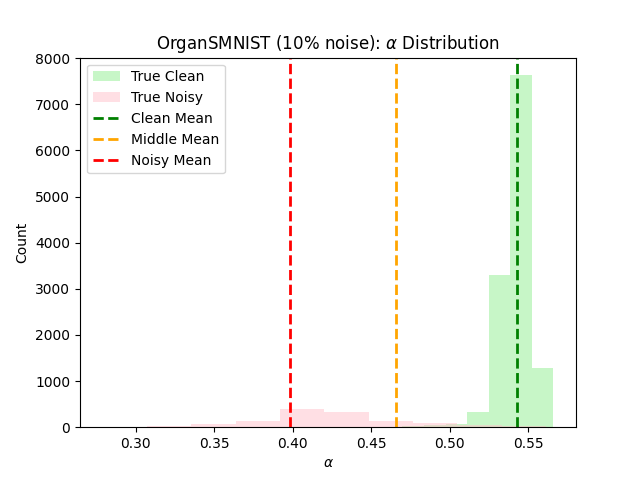}
\includegraphics[scale=0.3]{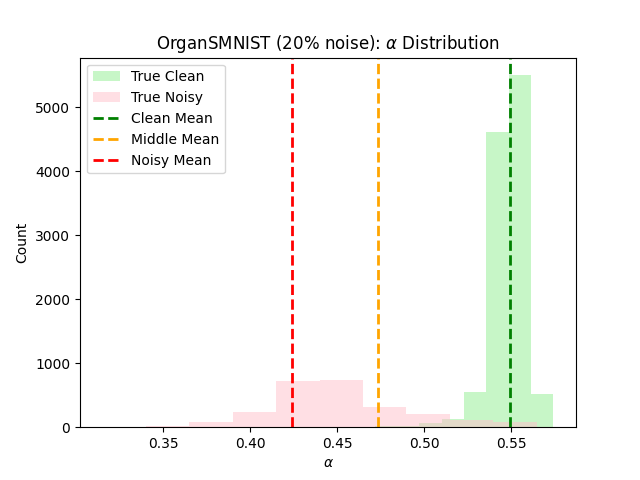}
\includegraphics[scale=0.3]{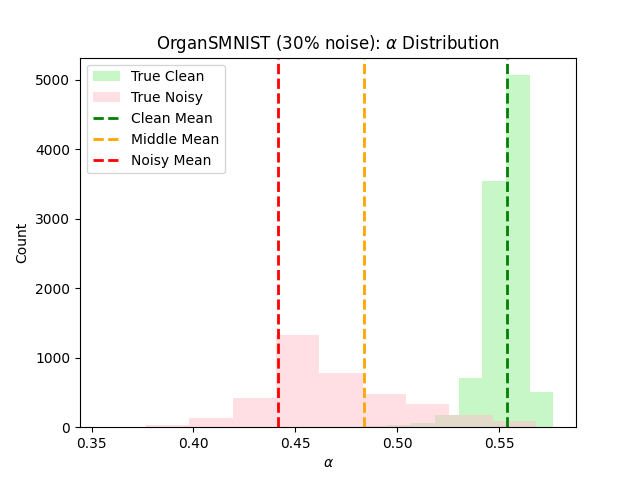}
\includegraphics[scale=0.3]{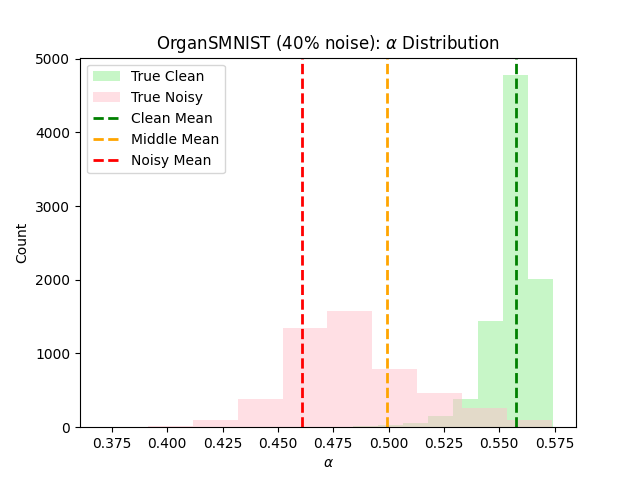}
\includegraphics[scale=0.3]{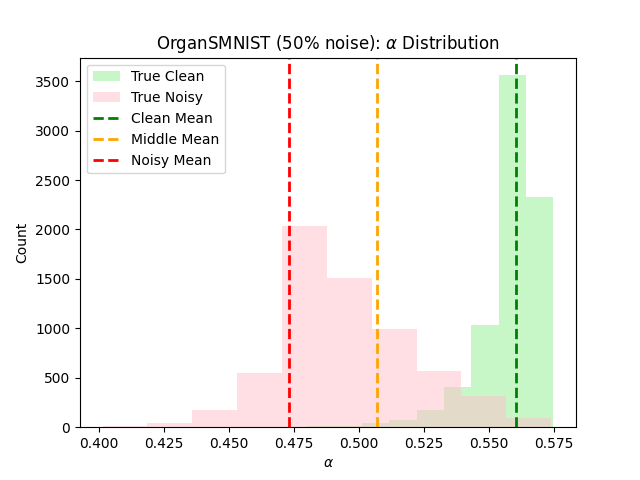}
\caption{Empirical separability of trust values $\alpha_i$ on OrganSMNIST under 10-50\% symmetric noise. Green / pink are true clean / noisy subsets, respectively, and the dashed lines indicate GMM component means (noisy / middle / clean).}
\label{fig:trust}
\end{figure*}

LiNC adds 3 main additional operations for total time
$T_{\text{LiNC}} = T_{\text{base}} + w\cdot\frac{N}{B}\cdot\mathcal{O}(B\,C) + \mathcal{O}(N) + \mathcal{O}(N\cdot T_{\text{fwd}}(1))$, as follows:\newline
\begin{enumerate}
    \item line 11: forming $q$ costs $\mathcal{O}(B\,C)$, $w$ times. Note that $w < E$.
    \item line 23: EM on $N$ scalars with $K=3$ and $I$ iterations costs $\mathcal{O}(NKI)=\mathcal{O}(N)$.
    \item line 26: computing $\hat{y}$ costs $\mathcal{O}(N\cdot T_{\text{fwd}}(1))$.
\end{enumerate}
This again gives us $\mathcal{O}(EN)$ runtime with LiNC.

\paragraph{Memory.} Standard training stores model parameters: $M_{\text{base}} = \mathcal{O}(|\theta|)$. 

LiNC in addition stores $N$ trust parameters: $M_{\text{LiNC}} = M_{\text{base}} + \mathcal{O}(N)$, which is negligible.

\section{Experimental Setup}
\label{sec:experiments}

\begin{table*}[h!]
\centering
\scriptsize
\setlength{\tabcolsep}{5pt}
\renewcommand{\arraystretch}{1.6}
\caption{AUC comparison of several noise detection baselines on the OrganSMNIST dataset with 20\% noise.}
\label{table:auc}
\begin{tabular}{lcccccccc}
\toprule
 & AUM & DataMaps & Data-IQ & EL2N & Forgetting & CNLCU-S & VoG & LiNC (ours) \\
\midrule
AUC & 0.8652 & 0.8351 & 0.8013 & 0.8547 & 0.6034 & 0.7878 & 0.9126 & \textbf{0.9837} \\
\bottomrule
\end{tabular}
\end{table*}

\begin{table*}[ht!]
\centering
\scriptsize
\setlength{\tabcolsep}{3pt}
\renewcommand{\arraystretch}{1.4}
\caption{Test top-1 accuracy on MedMNISTv2 at 0\% symmetric noise without LiNC. Baseline results are from~\citep{medmnistv2}. Two best results are in \textbf{bold}, with ViT-S/8-224 consistently achieving very high accuracy.}
\label{table:medmnist}
\resizebox{\linewidth}{!}{%
\begin{tabular}{lcccccccc}
\toprule
Dataset
& ResNet-18 (28) & ResNet-18 (224)
& ResNet-50 (28) & ResNet-50 (224)
& auto-sklearn & AutoKeras
& Google AutoML & ViT-S/8-224 \\
\midrule
PathMNIST      & 0.907 & 0.909 & \textbf{0.911} & 0.892 & 0.716 & 0.834 & 0.728 & \textbf{0.966} \\
DermaMNIST     & 0.735 & 0.754 & 0.735 & 0.731 & 0.719 & 0.749 & 0.768 & \textbf{0.864} \\
OCTMNIST       & 0.743 & 0.763 & 0.762 & \textbf{0.776} & 0.601 & 0.763 & 0.771 & \textbf{0.891} \\
PneumoniaMNIST & 0.854 & 0.864 & 0.854 & 0.884 & 0.855 & 0.878 & \textbf{0.946} & \textbf{0.934} \\
BreastMNIST    & \textbf{0.863} & 0.833 & 0.812 & 0.842 & 0.803 & 0.831 & 0.861 & \textbf{0.917} \\
BloodMNIST     & 0.958 & 0.963 & 0.956 & 0.950 & 0.878 & 0.961 & \textbf{0.966} & \textbf{0.989} \\
TissueMNIST    & 0.676 & \textbf{0.681} & 0.680 & 0.680 & 0.532 & \textbf{0.703} & 0.673 & 0.665 \\
OrganAMNIST    & 0.935 & \textbf{0.951} & 0.935 & 0.947 & 0.762 & 0.905 & 0.886 & \textbf{0.953} \\
OrganCMNIST    & 0.900 & \textbf{0.920} & 0.905 & 0.911 & 0.829 & 0.879 & 0.877 & \textbf{0.926} \\
OrganSMNIST    & 0.782 & 0.778 & 0.770 & 0.785 & 0.672 & \textbf{0.813} & 0.749 & \textbf{0.807} \\
\bottomrule
\end{tabular}
}
\end{table*}

\begin{table*}[ht!]
\centering
\scriptsize
\setlength{\tabcolsep}{2pt}
\renewcommand{\arraystretch}{1.5}
\caption{Test top-1 accuracy under symmetric noise for 10 2D MedMNISTv2 datasets. Standard training (No LiNC) and LiNC are evaluated at their Best and Last epochs.}
\label{table:medmnist_noise}
\resizebox{\linewidth}{!}{%
\begin{tabular}{l*{20}{c}}
\toprule
& \multicolumn{4}{c}{10\%}
& \multicolumn{4}{c}{20\%}
& \multicolumn{4}{c}{30\%}
& \multicolumn{4}{c}{40\%}
& \multicolumn{4}{c}{50\%} \\
\cmidrule(lr){2-5}\cmidrule(lr){6-9}
\cmidrule(lr){10-13}\cmidrule(lr){14-17}
\cmidrule(lr){18-21}
Dataset
& \multicolumn{2}{c}{No LiNC} & \multicolumn{2}{c}{LiNC}
& \multicolumn{2}{c}{No LiNC} & \multicolumn{2}{c}{LiNC}
& \multicolumn{2}{c}{No LiNC} & \multicolumn{2}{c}{LiNC}
& \multicolumn{2}{c}{No LiNC} & \multicolumn{2}{c}{LiNC}
& \multicolumn{2}{c}{No LiNC} & \multicolumn{2}{c}{LiNC} \\
\cmidrule(lr){2-3}\cmidrule(lr){4-5}
\cmidrule(lr){6-7}\cmidrule(lr){8-9}
\cmidrule(lr){10-11}\cmidrule(lr){12-13}
\cmidrule(lr){14-15}\cmidrule(lr){16-17}
\cmidrule(lr){18-19}\cmidrule(lr){20-21}
& Best & Last & Best & Last
& Best & Last & Best & Last
& Best & Last & Best & Last
& Best & Last & Best & Last
& Best & Last & Best & Last \\
\midrule
PathMNIST      & 0.9670 & 0.9536 & 0.9676 & 0.9154 & 0.9650 & 0.7960 & 0.9654 & 0.9331 & 0.9627 & 0.7055 & 0.9629 & 0.9055 & 0.9529 & 0.5864 & 0.9526 & 0.8852 & 0.9487 & 0.5057 & 0.9489 & 0.8033 \\
DermaMNIST     & 0.8430 & 0.8430 & 0.8530 & 0.8501 & 0.8154 & 0.7922 & 0.8342 & 0.8342 & 0.7988 & 0.7201 & 0.8266 & 0.8266 & 0.7790 & 0.6569 & 0.8132 & 0.8132 & 0.7529 & 0.5646 & 0.7934 & 0.7934 \\
OCTMNIST       & 0.9189 & 0.8155 & 0.9357 & 0.8399 & 0.9062 & 0.7363 & 0.9266 & 0.8724 & 0.8742 & 0.6777 & 0.9191 & 0.8350 & 0.8470 & 0.6642 & 0.8620 & 0.8386 & 0.8021 & 0.5953 & 0.8380 & 0.8217 \\
PneumoniaMNIST & 0.9010 & 0.8833 & 0.9355 & 0.9033 & 0.8717 & 0.8623 & 0.9183 & 0.9054 & 0.8690 & 0.7884 & 0.9386 & 0.8908 & 0.8496 & 0.8174 & 0.9194 & 0.8989 & 0.8493 & 0.7884 & 0.9013 & 0.8685 \\
BreastMNIST    & 0.9353 & 0.8979 & 0.9135 & 0.9135 & 0.9018 & 0.8583 & 0.8956 & 0.8956 & 0.8605 & 0.8331 & 0.8862 & 0.8566 & 0.7974 & 0.7500 & 0.8092 & 0.7952 & 0.7640 & 0.7148 & 0.8169 & 0.7913 \\
BloodMNIST     & 0.9832 & 0.9681 & 0.9881 & 0.9881 & 0.9829 & 0.9315 & 0.9844 & 0.9844 & 0.9761 & 0.8845 & 0.9818 & 0.9818 & 0.9757 & 0.8011 & 0.9769 & 0.9769 & 0.9673 & 0.7092 & 0.9739 & 0.9739 \\
TissueMNIST    & 0.6652 & 0.5475 & 0.6635 & 0.6532 & 0.6558 & 0.4992 & 0.6536 & 0.6441 & 0.6403 & 0.4359 & 0.6424 & 0.6339 & 0.6367 & 0.3979 & 0.6307 & 0.6307 & 0.6203 & 0.3445 & 0.6182 & 0.6145 \\
OrganAMNIST    & 0.9399 & 0.9150 & 0.9422 & 0.9422 & 0.9422 & 0.8444 & 0.9420 & 0.9342 & 0.9401 & 0.7596 & 0.9402 & 0.9169 & 0.9276 & 0.6488 & 0.9280 & 0.8938 & 0.9127 & 0.5602 & 0.9138 & 0.8769 \\
OrganCMNIST    & 0.9016 & 0.8880 & 0.9086 & 0.9086 & 0.8879 & 0.8424 & 0.8954 & 0.8927 & 0.8780 & 0.7743 & 0.8804 & 0.8690 & 0.8521 & 0.6875 & 0.8597 & 0.8366 & 0.8363 & 0.5896 & 0.8376 & 0.7945 \\
OrganSMNIST    & 0.7884 & 0.7774 & 0.7941 & 0.7941 & 0.7819 & 0.7282 & 0.7825 & 0.7802 & 0.7616 & 0.6563 & 0.7638 & 0.7504 & 0.7464 & 0.5882 & 0.7499 & 0.7228 & 0.7343 & 0.5097 & 0.7366 & 0.6845 \\
\bottomrule
\end{tabular}
}
\end{table*}

\subsection{Datasets and Noise Protocol}
We evaluate on ten 2D datasets from MedMNISTv2 \citep{medmnistv2} from varying domains including: PathMNIST, DermaMNIST, OCTMNIST, PneumoniaMNIST, BreastMNIST, BloodMNIST, TissueMNIST, OrganAMNIST, OrganCMNIST, and OrganSMNIST. Following common noisy-label evaluations \citep{han2018co,li2020dividemix}, we inject symmetric label noise at rates of $\rho\in\{0.1,0.2,0.3,0.4,0.5\}$ by randomly replacing a fraction $\rho$ of training labels with a uniformly sampled incorrect class.

\subsection{Model and Training}
We finetune an ImageNet-pretrained Vision Transformer \citep{dosovitskiy2020image} (ViT-S/8-224, frozen everywhere except classifier head and last transformer block) using Adam \citep{kingma2014adam} for a total of 70 epochs, with batch size 128, learning rate $5e-4$, and weight decay $1e-4$, along with a MultiStepLR scheduler that reduces the learning rate by a factor of 0.7 at 10, 20, 40, and 60 epochs. We report the best test accuracy and last test accuracy. 

LiNC uses soft warmup for $w=5$ epochs, soft correction for $s=5$ epochs, and hard correction for the remainder. Performance is relatively insensitive to the duration of the soft warmup and soft correction phases. In practice, around 5–10 epochs for the soft phases was sufficient to obtain stable, competitive results. The trust learning rate $\eta_\alpha=1$, and trust weight decay $\lambda_\alpha=1e-1$. All non-LiNC hyperparameters are shared between baseline and LiNC. The optimal hyperparameters without LiNC are still optimal with LiNC.

\subsection{Separation of Trust Parameters}
Figure~\ref{fig:trust} shows the empirical distributions of the learned trust parameters $\alpha_i$ after the five-epoch soft warm-up on OrganSMNIST. Across all evaluated noise rates, samples with clean observed labels are concentrated at higher trust values, whereas mislabeled samples tend to receive lower trust values. This behavior is consistent with Theorem~\ref{thm:grad_sign}: early in training, the model is more likely to agree with correctly labeled examples and disagree with incorrectly labeled examples, causing their trust parameters to move in opposite directions.

The separation remains visible as the symmetric noise rate increases from $10\%$ to $50\%$. At higher noise rates, the noisy distribution becomes broader and overlap near the boundary increases, but the low-, middle-, and high-mean GMM components remain ordered. The middle component provides a buffer between the two dominant modes, allowing LiNC to treat borderline samples as ambiguous rather than automatically correcting them. This is particularly important in medical imaging, where disagreement with an observed label may reflect diagnostic difficulty rather than an annotation error. By restricting correction to the lowest-mean component, LiNC avoids imposing a binary clean/noisy decision on every uncertain example.

Table~\ref{table:auc} provides a quantitative evaluation of this separation on OrganSMNIST with $20\%$ noise. LiNC achieves an AUC of $0.9837$, compared with $0.9126$ for VoG, the strongest competing method in this experiment. This corresponds to an absolute improvement of $0.0711$ in AUC. LiNC also substantially outperforms methods based on margins, forgetting events, losses, gradients, and other measures of sample difficulty. These results indicate that a single scalar learned within the ordinary training loop can provide a highly informative signal of possible label corruption without requiring a clean reference set, a second model, or a predefined trust threshold.

The magnitude of $\alpha_i$ should be interpreted as a model- and training-dependent trust score rather than as a probability that a label is correct. In practice, the learned scores and GMM assignments can support two uses: selecting the lowest-trust cases for correction during training and producing a ranked list of potentially problematic annotations for subsequent dataset review.

\begin{figure*}
\centering
\includegraphics[scale=0.36]{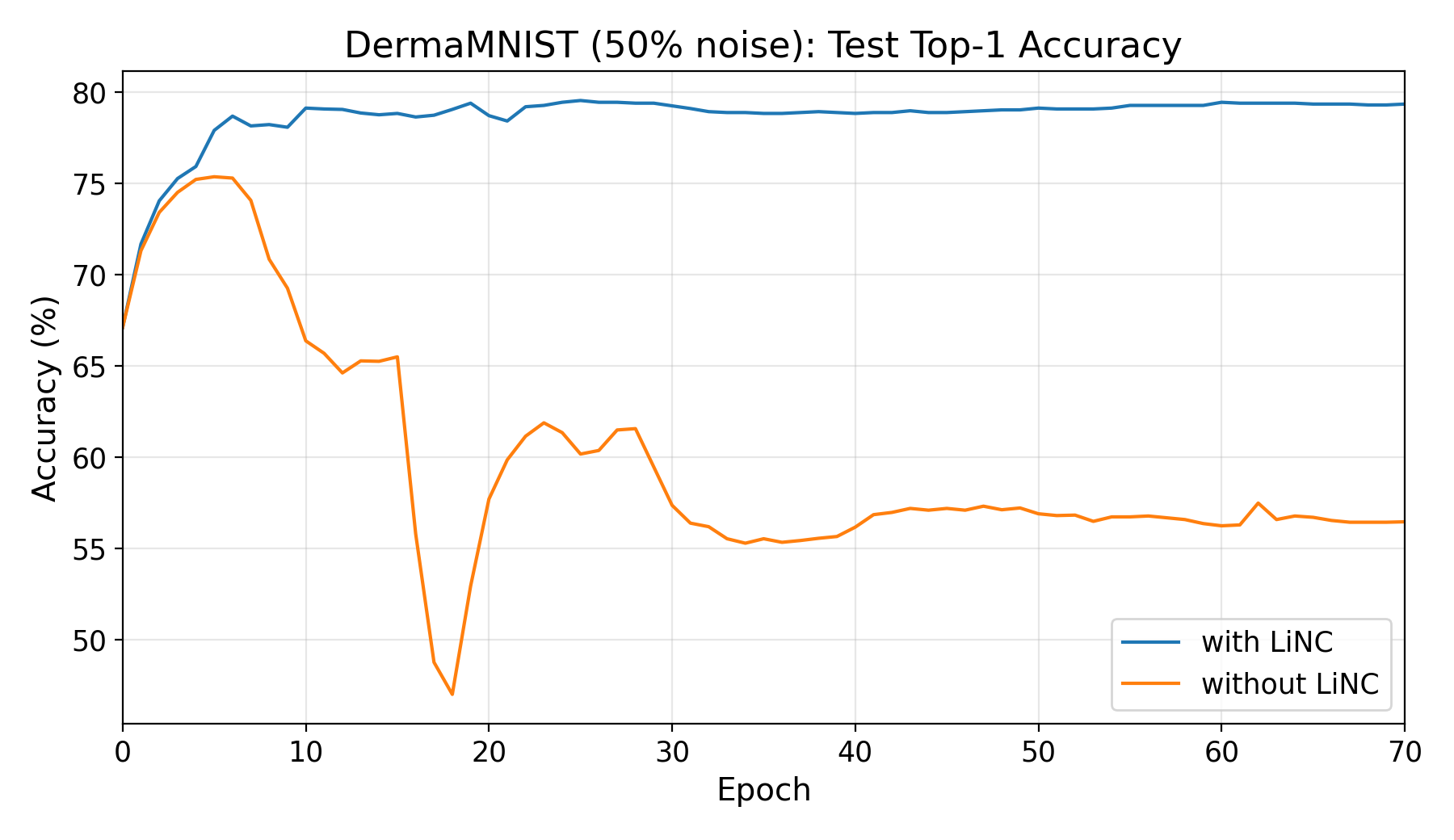}
\caption{Test top-1 accuracy over 70 epochs for DermaMNIST with 50\% symmetric noise. LiNC improves convergence and final accuracy compared to baseline training.}
\label{fig:plot}
\end{figure*}

\subsection{Effect on Downstream Performance}
Before introducing label noise, Table~\ref{table:medmnist} establishes that the selected ViT-S/8-224 model provides a strong reference model for the downstream experiments. It achieves the highest listed test accuracy on seven of the ten MedMNISTv2 datasets and remains competitive on the other three. The improvements in Table~\ref{table:medmnist_noise} therefore do not arise from comparing LiNC with an intentionally weak architecture. Moreover, LiNC and standard training use the same model, optimizer, learning-rate schedule, augmentations, and non-LiNC hyperparameters, isolating the effect of the proposed trust learning and correction procedure.

Figure~\ref{fig:plot} illustrates the training dynamics on DermaMNIST under $50\%$ symmetric noise. Standard training initially reaches a test accuracy of $0.7529$ but subsequently deteriorates to $0.5646$ as training continues. In contrast, LiNC reaches a higher best accuracy of $0.7934$ and maintains the same accuracy at the final epoch. The resulting improvement is therefore $4.05$ percentage points at the best epoch but $22.88$ percentage points at the last epoch. This trajectory is consistent with the role of LiNC: the method preserves useful early learning while preventing the later optimization process from increasingly fitting corrupted targets.

The same pattern appears across the ten datasets in Table~\ref{table:medmnist_noise}. Averaging across datasets, LiNC improves last-epoch accuracy by $2.19$, $7.86$, $12.31$, $16.94$, and $21.41$ percentage points at noise rates of $10\%$, $20\%$, $30\%$, $40\%$, and $50\%$, respectively. Thus, the benefit grows monotonically as label corruption becomes more severe. At $50\%$ noise, mean last-epoch accuracy increases from $0.5882$ without LiNC to $0.8023$ with LiNC. Improvements are not confined to a small subset of datasets: LiNC achieves higher last-epoch accuracy in 49 of the 50 dataset-noise combinations and higher best-epoch accuracy in 42 of 50 combinations.

There is a large difference between best- and last-epoch performance with and without LiNC. Averaged over all 50 experimental conditions, the difference between best- and last-epoch accuracy is $13.14$ percentage points for standard training but only $2.31$ percentage points for LiNC. At $50\%$ noise, this degradation grows to $23.06$ points without LiNC, compared with only $3.56$ points with LiNC. LiNC stabilizes generalization after the model would otherwise begin memorizing corrupted labels.

\section{Discussion}
\label{sec:discussion}
LiNC combines three simple mechanisms that address complementary aspects of noisy-label learning. First, the per-sample trust parameter provides a direct, differentiable representation of the model's agreement with each observed label. Second, the three-component GMM translates the resulting trust distribution into noisy, ambiguous, and clean groups without requiring a manually selected threshold or knowledge of the true noise rate. Third, the staged correction schedule limits the risk of immediately reinforcing incorrect model predictions. During soft correction, model outputs are used as distributions rather than hard class assignments and hard correction occurs only after this intermediate phase and only for samples assigned to the lowest-trust component.

The results suggest that LiNC is resistant to late-stage memorization. Improvements in best-epoch accuracy are comparatively modest, averaging $1.31$ percentage points across all conditions, whereas the average improvement at the last epoch is $12.14$ points. This difference is expected from the method's design. Early in training, both approaches can learn predictive structure from the clean portion of the data. Their behavior diverges later, when standard cross-entropy continues to optimize against corrupted labels while LiNC has identified and replaced many of those targets. The resulting stability is useful for model selection.

LiNC has the asymptotic runtime of standard training and requires only $\mathcal{O}(N)$ additional memory for the trust parameters. All trust updates occur within the same training loop, and the method does not require an ensemble, a second network, or a separate clean dataset. The learned trust values also provide an auditable output. For a medical dataset, low-trust cases could be prioritized for expert re-review, while cases assigned to the middle component could be examined as potentially ambiguous examples.

There are some limitations to this method. First, genuine clinical ambiguity may also lead multiple defensible labels to exist for the same case. Under these conditions, the learned trust score may capture a combination of annotation reliability, sample difficulty, and model noise rather than label noise alone. Evaluation on datasets with real-world label errors and inter-rater disagreement is therefore necessary.

Second, the method depends on the early-learning behavior underlying Theorem~\ref{thm:grad_sign}. A model can confidently disagree with a correct label or confidently agree with an incorrect label when the same systematic error is repeated throughout the dataset. This risk may be greater for rare classes or underrepresented patient groups. Future work should examine class-conditional and subgroup-specific trust distributions and measure whether correction rates or errors differ across clinically relevant groups.

Third, a three-component GMM will return three components even when the trust distribution contains little evidence of label corruption. This limitation is most apparent when the dataset is completely clean, but it may also matter under low noise, as suggested by the PathMNIST result at $10\%$. Future work should extend to allow the model to abstain from hard correction, select the number of mixture components adaptively, or send low-confidence cases to expert review.

\section{Conclusion}
\label{sec:conclusion}
We introduce LiNC, a lightweight method that learns a trust parameter for every training sample and uses the resulting distribution to separate noisy, ambiguous, and clean cases. LiNC achieves an AUC of $0.9837$ for detecting corrupted labels in the evaluated OrganSMNIST setting and improves mean last-epoch accuracy by $21.41$ percentage points across ten datasets at $50\%$ symmetric noise. It also substantially reduces the deterioration between peak and final performance, demonstrating generalization benefits. LiNC requires neither a clean validation set nor an additional model, preserves the asymptotic runtime of standard training, and adds only $\mathcal{O}(N)$ memory. Its learned trust values provide both a mechanism for targeted correction and an interpretable signal for dataset auditing. Future work should evaluate this framework under realistic, unknown noise processes and develop human-review mechanisms for human-model disagreement.


\bibliography{aaai2027}
\bibliographystyle{plainnat}

\end{document}